\documentclass[letterpaper]{article} 
\usepackage{aaai25}  
\usepackage{times}  
\usepackage{helvet}  
\usepackage{courier}  
\usepackage[hyphens]{url}  
\usepackage{graphicx} 
\usepackage{natbib}  
\usepackage{caption} 
\usepackage{algorithm}
\usepackage{algorithmic}

\usepackage{newfloat}
\usepackage{listings}
\DeclareCaptionStyle{ruled}{labelfont=normalfont,labelsep=colon,strut=off} 
\floatstyle{ruled}
\newfloat{listing}{tb}{lst}{}
\floatname{listing}{Listing}
\usepackage{tabularx}
\usepackage{svg}
\svgpath{{./figures}} 

\usepackage{xcolor}
\usepackage{subcaption}
\usepackage{amssymb}
\usepackage{amsmath}
\usepackage{xspace}
\newcommand{\name}{SD-SSM\xspace}
\usepackage{multirow}
\usepackage{amsthm}

\newtheorem{proposition}{Proposition}

\usepackage{booktabs}

\title{On the Expressiveness and Length Generalization of \\ Selective State-Space Models on Regular Languages}
\author{Aleksandar Terzić\textsuperscript{\rm 1, \rm 2}, 
Michael Hersche\textsuperscript{\rm 1}, 
Giacomo Camposampiero\textsuperscript{\rm 1, \rm 2}, \\Thomas Hofmann\textsuperscript{\rm 2}, 
Abu Sebastian\textsuperscript{\rm 1}, 
Abbas Rahimi\textsuperscript{\rm 1}}
\affiliations {
    \textsuperscript{\rm 1}IBM Research - Zurich\\
    \textsuperscript{\rm 2}ETH Zürich\\
    \{aleksandar.terzic1, michael.hersche, giacomo.camposampiero1\}@ibm.com, \\
    thomas.hofmann@inf.ethz.ch, \{ase, abr\}@zurich.ibm.com
}

\renewenvironment{abstract}
 {\small
  \begin{center}
  \bfseries \abstractname\vspace{-.5em}\vspace{0pt}
  \end{center}
  \list{}{%
    \setlength{\leftmargin}{0mm}
    \setlength{\rightmargin}{\leftmargin}%
  }%
  \item\relax}
 {\endlist}

\begin{document}

\maketitle

\begin{abstract}
\vspace{2mm} 
\begin{quote}

Selective state-space models (SSMs) are an emerging alternative to the Transformer, offering the unique advantage of parallel training and sequential inference. 
Although these models have shown promising performance on a variety of tasks, their formal expressiveness and length generalization properties remain underexplored.
In this work, we provide insight into the workings of selective SSMs by analyzing their expressiveness and length generalization performance on regular language tasks, i.e., finite-state automaton (FSA) emulation. 
We address certain limitations of modern SSM-based architectures by introducing the Selective Dense State-Space Model (\name), the first selective SSM that exhibits perfect length generalization on a set of various regular language tasks using a single layer.  
It utilizes a dictionary of dense transition matrices, a softmax selection mechanism that creates a convex combination of dictionary matrices at each time step, and a readout consisting of layer normalization followed by a linear map.
We then proceed to evaluate variants of diagonal selective SSMs by considering their empirical performance on commutative and non-commutative automata. 
We explain the experimental results with theoretical considerations.

\end{quote}
\end{abstract}

\begin{links}
     \link{Code}{https://github.com/IBM/selective-dense-state-space-model}
\end{links}

\section{Introduction}

Large language models (LLMs) are most often based on the Transformer architecture~\cite{vaswani_attention_2017}, a neural network that is highly parallelizable across a sequence of tokens. The parallelizability, coupled with hardware-aware implementations of the model \cite{dao2023flashattention2}, have allowed for efficient training over large corpora of long sequences.
However, despite empirical breakthroughs in natural language processing (NLP),
recent theoretical studies demonstrate that the Transformer has limited expressiveness, in particular when faced with state-tracking problems such as deciding the truth value of regular language expressions, i.e., emulating finite-state automata (FSA)~\cite{hahn_theoretical_2020, bhattamishra_ability_2020, merrill_parallelism_2023}. 

On the other hand, nonlinear recurrent neural networks (RNNs) can emulate any FSA; this can be seen by considering explicit mappings of FSA dynamics onto RNN weights, as surveyed in~\cite{svete2023efficient}. 
In practice, RNNs learn to emulate various FSA and often generalize to sequences much longer than those seen in training~\cite{deletang_neural_2023}.
However, in contrast to the Transformer, RNNs cannot be parallelized across the sequence length. 

Recently, a novel family of sequence models has emerged: linear state-space models (SSMs)~\cite{gu_efficiently_2022,gupta_diagonal_2022,smith_simplified_2023,orvieto_resurrecting_2023}.
SSMs provide an alternative sequence processing backbone that can be executed in parallel during training and sequentially during inference.
As a key driver for higher computational efficiency, many (selective) SSMs utilize diagonal rather than dense transition matrices~\cite{gupta_diagonal_2022, gu_parameterization_2022, smith_simplified_2023, orvieto_resurrecting_2023, gu_mamba_2023, de_griffin_2024}, allowing parallel scans for efficient training while remaining effective in many tasks of interest.
The most recent variants based on diagonal selective SSMs outperform the Transformer on several benchmarks, including language modeling~\cite{gu_mamba_2023}.

Although formal limits on the expressiveness of selective SSMs have recently been derived in the literature~\cite{zubic_limits_2024, orvieto_universality_2024, merrill_illusion_2024, cirone_theoretical_2024, sarrof2024expressivecapacitystatespace, grazzi2024unlocking}, the performance of selective SSMs on FSA emulation has not been sufficiently explored.
In this work, we experimentally and analytically study the capabilities of SSMs and selective SSMs to generalize to longer sequences than seen during training on a set of various FSA emulation tasks. 
Our contributions are as follows:

In Sec.~\ref{sec:SDSSM}, we introduce the first selective SSM capable of perfect ($\geq$ 99.9\%) length generalization on FSA emulation using a single layer.
We call this model \name, the \textit{Selective Dense State-Space Model}.
\name utilizes a dictionary of dense unstructured transition matrices, a softmax selection mechanism that creates a convex combination of a fixed number of transition matrices at each step, and finally applies a readout consisting of layer normalization followed by a linear map.
We identify that a common design choice, the presence of a nonlinear readout, prevents \name from achieving full accuracy on a challenging FSA emulation task.
We compare the model with the standard RNN and the LSTM~\cite{hochreiter_long_1997} in terms of the minimal sample length required to generalize in length and find that \name exhibits better length generalization.
Moreover, running \name with a parallel algorithm yields a notable speed-up over its sequential implementation.

In Sec.~\ref{sec:diagonal}, we take a closer look at selective SSMs with diagonal complex transition matrices.
We evaluate them on a set of FSA emulation tasks and analyze the effects of different architectural design choices on their performance on a commutative and non-commutative automaton. We find that perfect in-domain accuracy can be achieved on both automata, but length generalization is significantly worse on the non-commutative one.
We explain our experimental results with such systems by demonstrating that, under an assumption on the mapping of FSA to selective SSMs, single-layer selective diagonal SSMs are restricted to emulating commutative automata.

\section{Background}
\label{sec:background}

In this section we provide an overview of selective SSMs and FSA, and present an exact mapping of any FSA to the weights of a selective SSM.

\subsection{State-Space Models (SSMs) and Selective SSMs}

As their backbone, SSMs implement the standard linear time-invariant system of equations:
\begin{align} \label{eq_lssm}
x_{t} &= Ax_{t-1} + Bu_t \\
y_t &= Cx_t + Du_t
\end{align}
With $A \in \mathbb{R}^{n \times n}$, $B \in \mathbb{R}^{n \times d}$, $C \in \mathbb{R}^{d \times n}$ and $D \in \mathbb{R}^{d \times d}$
Since any real $n \times n$ matrix is diagonalizable up to an arbitrarily small perturbation of its entries, the above system can be equivalently represented using complex diagonal transition matrices~\cite{orvieto_resurrecting_2023}. The diagonal form is significantly more efficient to evaluate.
Because the system is linear in the hidden state $x_t$, the sequence $(x_1,...,x_T)$ can be computed using parallel algorithms~\cite{martin_parallelizing_2018, gu_combining_2021}.

Selective SSMs~\cite{gu_mamba_2023} implement the following system of equations:
\begin{align} \label{eq_slssm}
x_{t} &= A(u_t)x_{t-1} + b(u_t) \\
y_{t} &= c(x_{t}) + d(u_t) 
\end{align}
With $A(u_t) \in \mathbb{R}^{n \times n}$, $b:\mathbb{R}^d\rightarrow\mathbb{R}^n$, $c:\mathbb{R}^n\rightarrow\mathbb{R}^d$, and $d:\mathbb{R}^d\rightarrow\mathbb{R}^d$.
In contrast to standard SSMs, selective SSMs generate the matrix $A$ dynamically as a function of the input $u_t$. 
While most SSMs can be diagonalized, selective SSMs can only be diagonalized if all $A(u_t)$ matrices are simultaneously diagonalizable. This is a more restrictive condition than diagonalizability, as it requires that the product $A(u_t)A(u_{t'})$ commutes $\forall t, t' \in \left[1,T\right]$.
This system can also be evaluated in parallel using the parallel scan algorithm~\cite{blelloch_prex_1990,martin_parallelizing_2018,gu_mamba_2023}.

\subsection{Finite-State Automata (FSA)}
\label{subsec:semiauto}

A deterministic finite-state automaton (FSA) is an abstract model of computation defined as a 5-tuple $(Q, \Sigma, \delta, q_{\text{init}}, F)$, where $Q$ is a finite set of states, $\Sigma$ is a finite input alphabet, $\delta : Q \times \Sigma \rightarrow Q$ is the transition function, $q_{\text{init}} \in Q$ is a designated initial state, and $F \subseteq Q$ is the set of accepting states. 
In this work, we are not interested in the set $F$, and $q_{\text{init}}$ is only of limited interest. This leads us to the definition of a semi-automaton, which is a 3-tuple $(Q, \Sigma, \delta)$ with $Q$, $\Sigma$, and $\delta$ defined as above. 

A rich body of work connects semiautomata with algebraic semigroups~\cite{straubing_book, krohn_algebraic_1965, liu_transformers_2023, merrill_illusion_2024}.
A semigroup is a set with an associative binary operation defined on it.
Every semiautomaton induces a transformation semigroup consisting of a set of functions $\rho: Q \rightarrow Q$ defined for each $\sigma \in \Sigma$ by the transition function $\delta(\cdot, \sigma)$. The associative binary operation on this set is function composition. 
See, for example, the \emph{parity} automaton shown in Figure~\ref{fig:parity}.

\begin{figure}[t]
 \centering
 \includegraphics[width=0.8\linewidth]{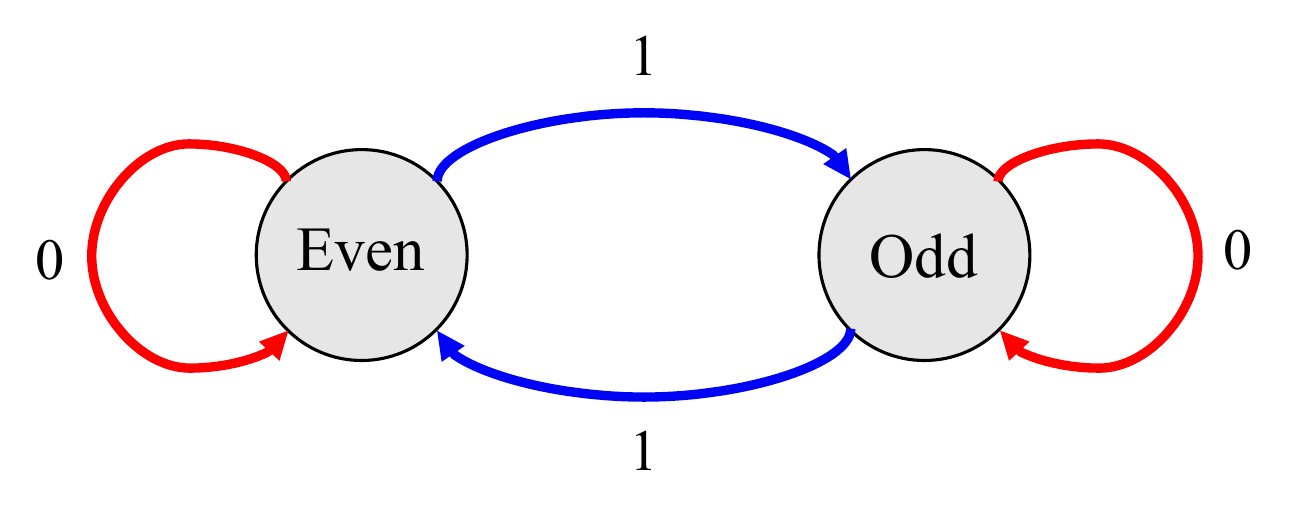}
 \caption{The parity automaton  with Q=\{\emph{Even}, \emph{Odd}\} and $\Sigma=\{0,1\}$. The automaton starts in the \emph{Even} state, toggles on input 1, and makes no transition on input 0.}
 \label{fig:parity}
 \end{figure}

The corresponding transformation semigroup consists of two elements, these being the transition functions corresponding to the two inputs, $\delta(\cdot, 0)$ and $\delta(\cdot, 1)$. 
The transition function $\delta(\cdot, 0)$ corresponds to the identity operation, since no matter in which state $q$ the automaton is in, $\delta(q, 0) = q$. 
If the states are one-hot encoded, $\delta(\cdot, 0)$ is equivalently to the $2\times 2$ identity matrix. 
Meanwhile, $\delta(\cdot, 1)$ can be equivalently represented as a $2\times 2$ matrix with zeros on the diagonal and ones in the off-diagonal entries. This matrix \emph{toggles} the one-hot encoded state. 
Using matrix representation, the function composition can be equivalently represented as matrix multiplication. Therefore, the final state of the automaton can be obtained by evaluating the chain of matrix products corresponding to the given sequence of inputs.
For a deeper discussion of the topic, we recommend~\cite{liu_transformers_2023}.

\subsection{Mapping an FSA to a Selective SSM}

Any FSA can be mapped to a selective SSM.
To see this, we show a mapping procedure that is conceptually equivalent to those of ~\cite{merrill_illusion_2024, liu_transformers_2023}. 
For each $q \in Q$, encode it using $enc: Q \rightarrow\mathbb{R}^{|Q|}$ such that the encodings of different states are orthogonal. Orthogonality is a sufficient, but not a necessary, condition for mapping an FSA to a selective SSM.
Given the encoding of the states, we can now map the transition function $\delta : Q \times \Sigma \rightarrow Q$ to the transition matrices $A(u_t)$ from Eq.~\eqref{eq_slssm}. 
In this mapping, each symbol in the input alphabet $\sigma \in \Sigma$ has an associated transition matrix $A(\sigma)$ defined by the transition function $\delta(\cdot, \sigma): Q \rightarrow Q$.
The mapping between this function and $A(\sigma)$ is defined via the sum $A(\sigma)=\sum_{q \in Q}enc(\delta(q, \sigma))\cdot enc(q)^T$.

We now have all the ingredients needed to map any FSA to Eq.~\eqref{eq_slssm}. 
This is achieved by setting $x_0 = q_{init}$, identifying the inputs $u_t$ as elements of the alphabet $\Sigma$ and thus setting $A(u_t) = A(\sigma)$ as above, and setting $B=0$. By induction, one can see that $x_t = enc(q_t)$ with $q_t$ achieved by $t$-fold repeated application of the transition function $\delta$ onto $q_{init}$ given a sequence of inputs $(\sigma_1 , ..., \sigma_t)$.

\begin{table*}[t]
\centering
\resizebox{\textwidth}{!}{
\begin{tabular}{lcccccccc}
\toprule
Task        & RNN       &         Transformer   &        S4D      & H3    &     Mamba  &      RegularLRNN  & $\mathbb{C}$ Diagonal   &  \name (ours)  \\ \cmidrule(r){1-1} \cmidrule(r){2-2} \cmidrule(r){3-3} \cmidrule(r){4-4} \cmidrule(r){5-5} \cmidrule(r){6-6} \cmidrule(r){7-7} \cmidrule(r){8-8} \cmidrule(r){9-9}
\multicolumn{9}{l}{\cite{deletang_neural_2023}} \\
Parity      & 100\space\space/ 100 & 52.3 / 50.4 & 50.1 / 50.0  &  50.0 / 50.0    & 50.3 / 50.1 & 100\space\space/ 100  &  99.3 / 72.4   & 100\space\space/ 100                \\
Even Pairs  & 100\space\space/ 100 & 100 / 100   & 50.4 / 50.3  &  51.0 / 50.5 &  100\space\space/ 100 & 100\space\space/ 100   & 54.5 / 54.3 & 100\space\space/ 100                  \\
Cycle       & 100\space\space/ 100 & 73.6 / 52.9 & 33.6 / 29.2  & 20.1 / 20.0 & 21.1 / 21.0   & 100\space\space/ 100  & 99.6 / 90.4  &  100\space\space/ 100                    \\
Arithmetic  & 100\space\space/ 100 & 25.5 / 23.5 & 20.1 / 20.0  & 20.1 / 20.0    & 20.1 / 20.1 & 33.3 / 30.2 &  22.1 / 21.7  & 99.9 / 98.5  \\ \cmidrule(r){1-1}
\multicolumn{9}{l}{\cite{liu_transformers_2023}} \\
$C_2 \times C_4$   & 100 / 100 & --- & ---  & ---  & --- & 100\space\space/ 99.4 & 79.2 / 60.4 & 100 / 93.3   \\
$D_4$   & 100 / 100 & --- & ---  & ---  & --- & 100\space\space/ 100  & 32.6 / 29.8 & 99.9 / 99.9   \\
$A_5$   & 100 / 100 & --- & ---  & ---  & --- & 100\space\space/ 100  & 8.3 / 8.2 & 100\space\space/ 100 
\\ \bottomrule       
\end{tabular}
}
\caption{Maximum/average length generalization accuracy (\%) on FSA emulation tasks over three random seeds using single-layer models, except the Transformer, which uses five layers. The Transformer results are taken from \cite{ruoss_randomized_2023} and use randomize RoPE positional encodings. While they evaluate on sequences of length 50 to 500 and we evaluate on length 1 to 500, the failure of the model is still evident. We denote the selective SSM defined by~\cite{fan_advancing_2024} as RegularLRNN. The complex ($\mathbb{C}$) diagonal model is a diagonal selective SSM which we defined in Sec.~\ref{sec:diagonal}. Our \name achieves near-perfect average accuracy on all investigated automata.}
\label{tab:main_table}
\end{table*}

\section{Experimental Analysis of SD-SSM on Regular Languages}
\label{sec:SDSSM}

This section presents our first contribution. 
We start with an empirical study of different sequence models on various FSA emulation tasks. 
The models are evaluated in terms of their length generalization capabilities. 
We then propose a novel selective SSM, SD-SSM, that successfully learns to emulate the dynamics of complex FSAs using a single layer. 

\subsection{Task Description and Experimental Setup}

The investigated tasks involve tracking the state transitions of different FSAs. The experimental code is based on~\cite{deletang_neural_2023} and~\cite{liu_transformers_2023}.
We evaluate our models on seven different FSAs. Parity, Even Pairs, Cycle, and Arithmetic are taken from~\cite{deletang_neural_2023}. 
We further define three automata based on the Cayley diagrams of different algebraic groups. $C_2\times C_4$, the direct product of cyclic groups $C_2$ and $C_4$, is a commutative, solvable group with eight states. $D_4$, the dihedral group with eight elements, is a non-commutative, solvable group. $A_5$ is the group of even permutations of five elements, a non-solvable group with 60 states.
The tasks are described in Appendix~A\footnote{The appendix is available in the preprint~\cite{terzic_sd_ssm_2024}.}.

At each training step, we uniformly sample a sequence length $l$ between 1 and the maximum training length $L$. We then generate a random input sequence $(\sigma_1,\dots,\sigma_l)$ and use it to emulate the automaton.
The models are trained to minimize the cross-entropy loss between their output at the final step and the final state of the emulated automaton.
As in~\cite{deletang_neural_2023}, we train the model for a fixed number of steps and report the test accuracy of the model at the final training step.
The hyperparameters for reproducing the experiments are reported in Appendix~B.

\begin{figure}
\centering
\includegraphics[width=1.0\linewidth]{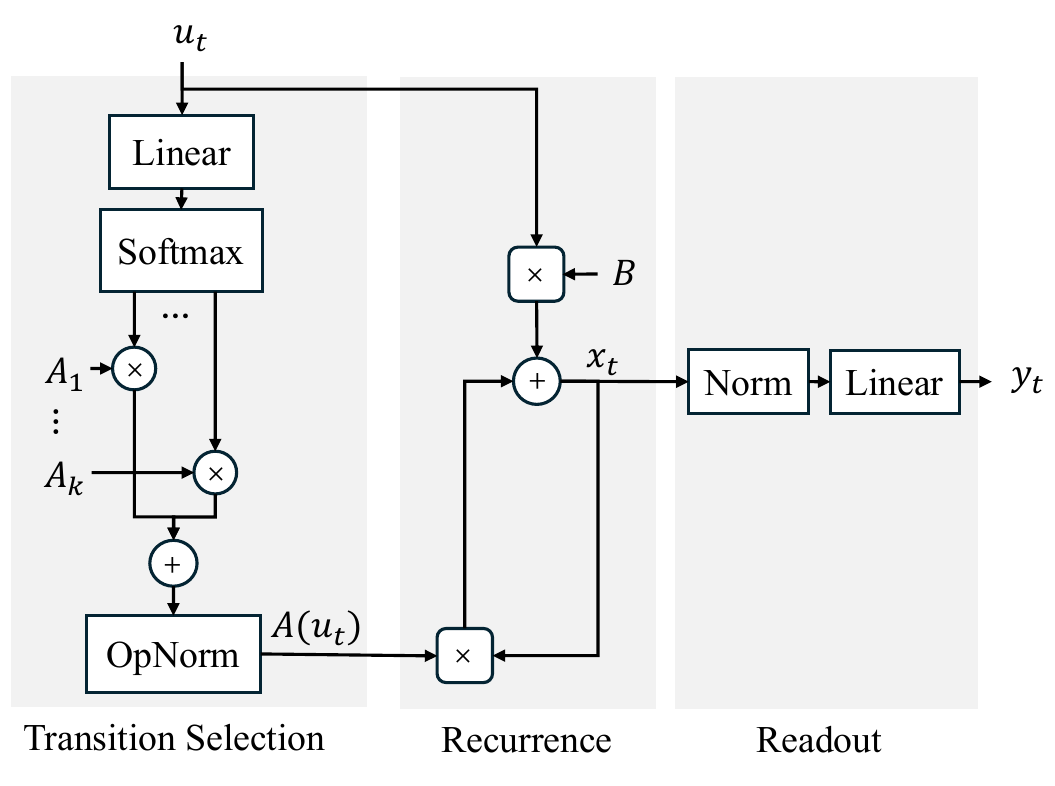}
\caption{The \name model consists of three main steps. Firstly, the dense transition matrices for all time steps are generated using a softmax selection mechanism and operator normalization inspired by~\cite{fan_advancing_2024}. The second step is the linear recurrence from Eq.~\eqref{eq_slssm}. The final step is the readout, consisting of layer normalization followed by a linear map.}
\label{fig:model_sketch}
\end{figure}

\subsection{Modern SSMs Fail to Emulate FSAs}

We start the discussion by evaluating prominent sequence models from the literature on various tasks using the experimental setup described above.
Our evaluation includes the standard nonlinear RNN~\cite{elman_finding_1990}, the Transformer~\cite{vaswani_attention_2017,ruoss_randomized_2023}, S4D~\cite{gu_parameterization_2022}, H3~\cite{fu_hungry_2023}, Mamba~\cite{gu_mamba_2023}, as well as the block-diagonal selective SSM RegularLRNN~\cite{fan_advancing_2024}.
All models are trained on sequences of length up to 40 and are tested on sequences of length 1 up to 500. We train the models using $3$ random seeds and report the maximum/average area under the accuracy vs. length curve. 
Except for the Transformer, we report all results using a single layer.

As shown in Table~\ref{tab:main_table}, none of the aforementioned models achieve perfect length generalization on all of the tasks. 
In fact, the models often fail already on in-domain lengths. 
As a particularly important example consider Mamba, the most recently investigated model and one that has shown the most promise as an alternative to the Transformer as LLM backbone. 
On the Arithmetic task, the best Mamba model achieves an in-domain accuracy of $91.6\%$. 
It exhibits an accuracy of below $99\%$ for the first time with sequences of length 22. The performance then rapidly drops with longer sequences, yielding the low accuracy (20.1\%) reported in Table~\ref{tab:main_table}.
We ran hyperparameter search in which we varied the state size, the learning rate, the weight decay factor, and trained for $10^6$ steps with a batch size of 128. The Mamba model described above is the best out of 96 runs.

Additional results with two layers of S4D, H3, Mamba, as well as S4~\cite{gu_efficiently_2022} and Hyena~\cite{poli_hyena_2023} are reported in Appendix~C. While the models often exhibit better length generalization with two layers than with a single one, none of them achieve significant length generalization on FSA emulation.

\subsection{\name Learns to Emulate FSA}

We now present an architecture that successfully learns to emulate a wide range of different FSA dynamics using a single layer.
We call the model \name, standing for \emph{Selective Dense State-Space Model}. The model architecture is presented in Figure \ref{fig:model_sketch}. 
It can be conceptually separated into three different phases:
\begin{align} 
& A(u_t) = \text{OpNorm}(\sum_{i=1}^k \text{softmax}(Su_t)[i] A_i )\\
& x_{t+1} = A(u_t)x_t + Bu_t \\
& y_{t} = C(\text{LayerNorm}(x_t))
\end{align}
In the first phase, the transition matrices $A(u_t)$ are generated for each input $u_t$ in parallel. This is achieved by passing the input embeddings through a linear layer ($S$) and then processing the resulting vector using the softmax function. 
The outputs of the softmax are used to weigh a dictionary of $k$-many trainable dense transition matrices labeled $A_1,...,A_k$. In our experiments, we use between 5 and 20 transition matrices (see Appendix B).

The weighted matrices are summed together and then modified using an operator normalization procedure (OpNorm).
Operator normalization is required for the stability of the system, preventing the eigenvalues of the transition matrices from growing beyond 1. 
The avenue we pursue is heavily inspired by the concurrent work of~\cite{fan_advancing_2024}, which normalizes the generated $A(u_t)$ matrices before applying the recursion in Eq.~\eqref{eq_slssm}. Their normalization scheme consists of normalizing the columns of $A(u_t)$ by setting each column vector $a_i$ thereof to $a_i / \text{max}(1, l_p(a_i))$, with $l_p(\cdot)$ denoting the standard $l_p$-norm operation. 
We adopt a version of the column-wise normalization scheme in our \name.
Concretely, we divide each column of $A$ by its $l_p$-norm, $a_i \leftarrow a_i / l_p(a_i)$.
While~\cite{fan_advancing_2024} finds that $p=1.2$ works well across all tasks, we varied $p \in [1.0,1.5]$ across the tasks via a hyperparameter search. 

The final two phases consist of executing the recurrence in Eq.~\eqref{eq_slssm}, followed by a readout of the state value. The readout consists of Layer Normalization~\cite{ba_layer_2016} followed by a linear layer. 
We find the design of the readout to be especially important for length generalization. In the experiments that we present in the following subsection, we see that the typically used MLP readout, such as what~\cite{fan_advancing_2024} used, has a negative impact on the generalization properties of our model.

We compare our matrix generation with RegularLRNN, which generates transition matrices as $A(u)=W_2\sigma(W_1u)$. Their transition matrices are block-diagonal, and we assume that the block size equals the square root of the state dimensionality, as was also the case in their experiments. 
The total number of parameters in their transition matrix generator is then $d n  \sqrt n + n^3$, while our generation incurs a cost of $k(d + n^2)$. For a fixed $k$, our generation method is more parameter efficient.

The results on FSA emulation using a one-layer \name are reported in  Table~\ref{tab:main_table}. As shown, on all of the tasks we investigate, the best \name achieves perfect ($\geq$ 99.9\%) accuracy using only one layer. 
We do observe that \name exhibits higher variability on the $C_2 \times C_4$ automaton compared to RegularLRNN. RegularLRNN with one layer does however not perform well on Arithmetic.

\subsection{Nonlinear Readout Hurts State Tracking}
We additionally ablate the \name readout. Apart from the desire for simplicity, the \name's readout also emerged from the observation that a more complex readout has detrimental effects on the length generalization.
On the Arithmetic task, we conducted extensive experiments in which the linear layer in \name's readout was replaced by a standard two-layer MLP with the ReLU non-linearity. The hidden layer of the MLP was configured to consist of 64 units, equal to the state size of the model.
We varied the learning rate in \{2e-5, 1e-4, 5e-4\}, the $p$ parameter in $l_p$ normalization in \{1.1, 1.2, 1.3\}, and we experimented with two regularization techniques, weight decay in \{0, 1e-4, 1e-3\} and dropout in \{0.1, 0.2, 0.5\} on the intermediate activations of the MLP readout.
The best accuracy was achieved by using weight decay. The model achieves an accuracy of $71.9\%$, significantly below $99.9\%$ achieved by the linear readout \name.
Further results are shown in Appendix~C.

\subsection{Length Efficiency Analysis}

We further consider a more demanding experimental setup. 
We compare how well RNN, LSTM, and SD-SSM extrapolate to longer sequences when trained on very short sequences. 
Concretely, the initial state of the network is chosen uniformly at random from the set of automaton states, and the model is trained to predict the automaton state after consuming an input of a short length, up to 8.
As the training sequences are very short, we observe overfitting: the training loss will often notably increase after having plateaued at a low value. Thus, we validate the models on sequences up to length 40 and report the accuracy obtained on sequences up to length 500.
Further details on this experimental setup are outlined in Appendix~B. 
Table~\ref{tab:short-training-sequency-A5} shows favorable results for \name in this challenging task. For very short sequences, it exhibits better length generalization compared to the RNN and the LSTM with the same state size. Further results are provided in Appendix~C.

\begin{table}
\centering
\resizebox{0.45\textwidth}{!}{
\begin{tabular}{cccccc}
\toprule
       & \multicolumn{5}{c}{Training Length}    \\  \cmidrule(r){2-6}
Model  &    4       &    5      &   6        & 7          & 8             \\ \cmidrule(r){1-1} \cmidrule(r){2-6}
RNN    &  7.1     &  26.1   &  85.4    &     99.4     &   99.9      \\ \cmidrule(r){1-1}
LSTM   &    14.7   &  66.9   &  97.6   &  99.9   &  100            \\ \cmidrule(r){1-1}
SD-SSM (ours) &    31.5  &   83.3     &  97.2    &  99.4    &  100        \\ 
\bottomrule          
\end{tabular}
}
\caption{Maximum length generalization accuracy (\%) on sequences up to length 500 over three random seeds. The models were trained to emulate the $A_5$ automaton with very short sequences (4 to 8) using a state size of 128.}
\label{tab:short-training-sequency-A5}
\end{table}

\subsection{\name Can Leverage Parallel Scans}

\name's linear recurrence over the hidden state allows us to use the parallel scan algorithm to compute the sequence of states $x_1,\dots,x_t$. Table~\ref{tab:time-measurements} reports the combined forward and backward pass runtime (in seconds) of a single layer \name on an NVIDIA V100 with different sequence lengths (L), a state size of 64, and a batch size of 16 in PyTorch. We use the implementation of the parallel scan algorithm from~\cite{fan_advancing_2024}. The parallel algorithm allows \name to be trained more efficiently than in sequential mode, despite its use of unstructured matrices.

\begin{table}
\centering
\begin{tabular}{ccccc}
\toprule
& \multicolumn{4}{c}{Sequence Length}    \\  \cmidrule(r){2-5}
Compute Mode    &    64       &     128         &   256       & 512             \\  \cmidrule(r){1-1} \cmidrule(r){2-5}
Recurrent       &  1.82\,s       &     6.40\,s        &  21.93\,s      &    77.27\,s          \\ \cmidrule(r){1-1} 
Parallel        &  1.91\,s       &     4.55\,s        &  11.67\,s      &    26.68\,s             \\
\bottomrule          
\end{tabular}
\caption{Forward + backward pass runtime (seconds) required to evaluate one batch (16) with \name using a sequential or a parallel algorithm. The model uses one layer and a state size of 64.}
\label{tab:time-measurements}
\end{table}

\section{A Limitation of Diagonal Selective SSMs}
\label{sec:diagonal}

To better understand the need for dense transition matrices, we start this section by evaluating diagonal selective SSMs on the previously used set of regular language tasks. We observe that it exhibits significantly lower scores than its dense counterpart, \name.
We then take a closer look at the performance of different architectural variants of a diagonal selective SSM on two selected automata, one being commutative and the other being non-commutative with respect to the inputs.
We find that diagonal selective SSMs tend to perform significantly better on the commutative automaton.
We explain our experimental findings by drawing a connection between the presented FSA$\rightarrow$selective SSM mapping and results from linear system theory. 

\subsection{Diagonal Selective SSMs Fail to Emulate FSAs}

We first present and evaluate a selective SSM that utilizes complex-valued diagonal transition matrices on the full set of FSA tasks, motivated by the prominence of such models in modern literature~\cite{gupta_diagonal_2022,orvieto_resurrecting_2023}.
The models we use in our experiments are variations on the following, general model:
\begin{align} 
& \tilde{A}_{ \operatorname{Re},  \operatorname{Im}}(u_t)= W^o_{ \operatorname{Re},  \operatorname{Im}}(\sigma(W^i_{ \operatorname{Re},  \operatorname{Im}}(u_t))) \\
& \tilde{A}(u_t)[m] = (\tilde{A}_{ \operatorname{Re}}(u_t)[m] + i\tilde{A}_{ \operatorname{Im}}(u_t)[m]) \\
& A(u_t)[m] = \beta \tilde{A}(u_t)[m] / |\tilde{A}(u_t)[m]| \\
& x_{t+1} = A(u_t)\odot x_t + Bu_t \\
& y_{t} = W_o^y(\sigma(W_i^y( \operatorname{Re}(x_t)\oplus  \operatorname{Im}(x_t)))
\end{align}
with $W_{Re,Im}^i \in \mathbb{R}^{n\times d}$, $W_{Re,Im}^o \in \mathbb{R}^{n\times n}$, $W_{i}^y \in \mathbb{R}^{2n\times 2n}$, $W_{o}^y \in \mathbb{R}^{d\times 2n}$,  $B \in \mathbb{C}^{n \times d}$, $\beta \in [0,1]$, $\odot$ denoting the element-wise product, $\sigma(\cdot)$ an element-wise activation function, in our case ReLU, $|\cdot|$ the element-wise absolute value function, and $a \oplus b$ denoting the concatenation of the two vectors $a$ and $b$.

The entries along the diagonal of the presented model's transition matrices are complex numbers whose real and imaginary parts are generated as functions of the input using a two-layer MLP.
They are constrained to be on the complex circle of radius $\beta$, where $\beta$ is a hyperparameter. 

At the readout, we concatenate the real and the imaginary parts of the state vector and then further process this vector using a two-layer MLP. While the nonlinear readout was detrimental in \name, we find that the $\mathbb{C}$ diagonal model exhibits better accuracy with a nonlinear readout instead.

We evaluate the $\mathbb{C}$ Diagonal model on the previously used tasks and report the results in Table~\ref{tab:main_table}. 
On Parity and Cycle, the model achieves almost perfect length generalization on the evaluated sequences, significantly better than Mamba which uses real-valued diagonal transition matrices. However, on Even Pairs, the $\mathbb{C}$ diagonal model achieves perfect accuracy on in-domain lengths, but fails drastically as soon as the sequence length is increased beyond the training domain. On Arithmetic, it fails on in-domain lengths, dropping below $99\%$ already at sequence length 8. 

\begin{figure}[t]
\centering
\begin{subfigure}{.2\textwidth}
  \includegraphics[width=.9\linewidth]{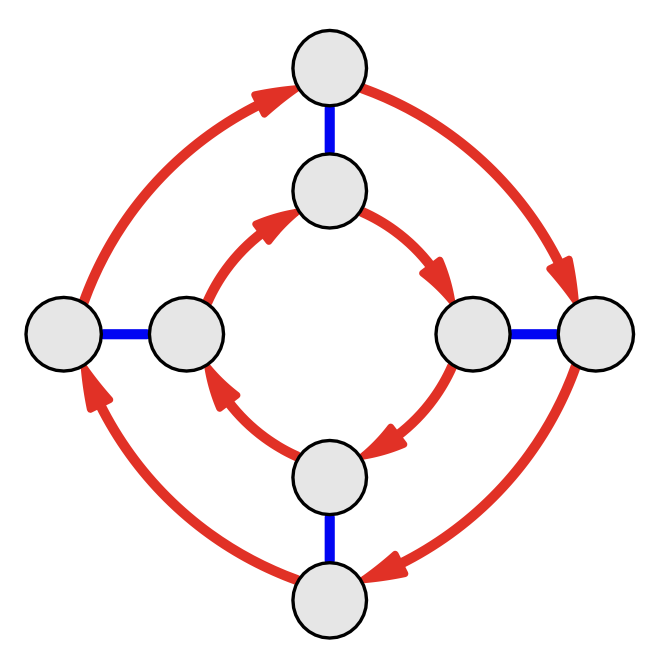}
  \caption{$C_{2\times4}$ Automaton}
  \label{fig:C24_FSA}
\end{subfigure}%
\begin{subfigure}{.2\textwidth}
  \centering
  \includegraphics[width=.9\linewidth]{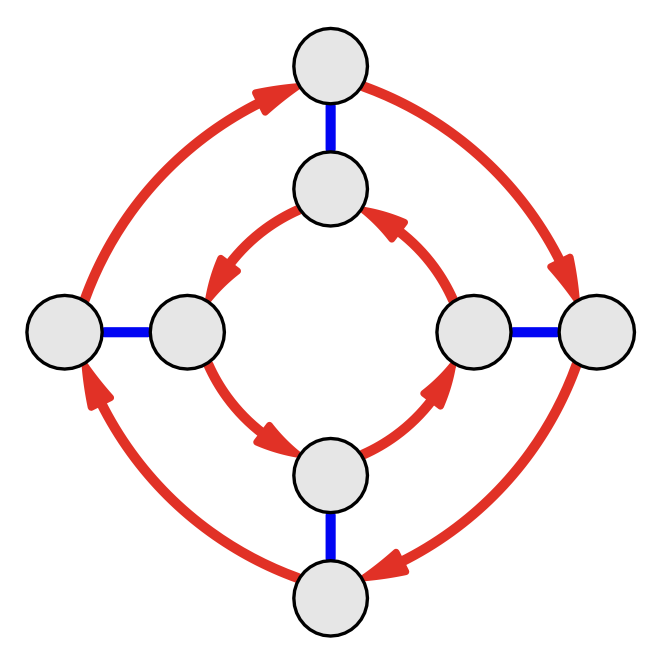}
  \caption{$D_4$ Automaton}
  \label{fig:D4_FSA}
\end{subfigure}
\caption{Cayley diagrams of $C_{2}\times C_4$ and the $D_4$ automata, both with two actions: \emph{toggle} (blue) and \emph{move} (red)~\cite{carter_visual_2009}. $C_{2} \times C_{4}$ is commutative. Starting at any state, applying \emph{toggle} followed by \emph{move} results in the same state as \emph{move} followed by \emph{toggle}. The same does not hold for $D_4$.}
\label{fig:C24_vs_D4_Cayley}
\end{figure}

\subsection{Performance on Commutative and Non-Commutative Automata}

We investigated in more depth the behavior of diagonal selective SSMs on the $C_{2}\times C_{30}$ and $D_{30}$ automata. Both $C_{2}\times C_{30}$ and $D_{30}$ are solvable groups, but $C_{2}\times C_{30}$ is commutative and $D_{30}$ is not. According to recent theoretical results on the expressive capacity of diagonal selective SSMs~\cite{merrill_illusion_2024, sarrof2024expressivecapacitystatespace}, both can be emulated by diagonal selective SSMs. Smaller versions of the automata, each with 8 instead of 60 states, are shown in Figure~\ref{fig:C24_vs_D4_Cayley}. 
As the automata have long diameters, i.e., the expected number of random actions required to visit each state of the automaton is large, we train the models on these two tasks with sequences up to length 90 and report the average accuracy on sequences up to length 600.

We train four different complex diagonal selective SSM variants on $C_2 \times C_{30}$ and $D_{30}$ FSA emulation.
The two central architectural choices we ablate are the use of the $B$ matrix in the transition as well as the use of a nonlinear readout.
We report the results obtained with the best-performing seed for each model in Table~\ref{tab:C2xC30_vs_D30}. 

\begin{table}[t!]
\resizebox{0.47\textwidth}{!}{
\centering
\begin{tabular}{cccccc}
\toprule
  & \multicolumn{4}{c}{Variants of $\mathbb{C}$ Diagonal} & \name  \\ \cmidrule(r){1-1} \cmidrule(r){2-5} \cmidrule(r){6-6}
$B = 0?$ & \multicolumn{2}{c}{Yes} & \multicolumn{2}{c}{No} &  \multicolumn{1}{c}{No} \\ 
\cmidrule(r){1-1} \cmidrule(r){2-3} \cmidrule(r){4-5} \cmidrule(r){6-6}
Readout            & \multicolumn{1}{c}{Linear}     & \multicolumn{1}{c}{Nonlin.}      & \multicolumn{1}{c}{Linear}       & \multicolumn{1}{c}{Nonlin.}    & \multicolumn{1}{c}{Linear}     \\ \cmidrule(r){1-1} \cmidrule(r){2-2} \cmidrule(r){3-3} \cmidrule(r){4-4}  \cmidrule(r){5-5} \cmidrule(r){6-6} 
   $C_2 \times C_{30}$   & 100        & 87.6         & 65.8        & 81.7   &   100      \\     
   $D_{30}$   & 8.35        & 8.35         & 11.3         & 61.0       &    100           \\    
   \bottomrule
\end{tabular}
}
\caption{Maximum length generalization accuracy (\%) of variants of the $\mathbb{C}$ diagonal selective SSM and our \name.}
\label{tab:C2xC30_vs_D30}

\end{table}

Firstly, we observe that all models learn to emulate $C_2 \times C_{30}$ perfectly on in-domain lengths. Their length generalization is however significantly affected by the architectural choices. Models that do not utilize the $B$ matrix tend to learn solutions that exhibit better length generalization than their counterparts which include the $B$ matrix.

The results are significantly different on $D_{30}$. Without the $B$ matrix, the models completely fail to learn the dynamics of the automaton, exhibiting very low in-domain accuracy.
The model exhibits significantly better length generalization once the $B$ matrix is introduced, although it is only with a nonlinear readout that the model learns to emulate the automaton even on in-domain lengths. The best length generalization accuracy on $D_{30}$ is significantly lower than what could be achieved on the $C_2 \times C_{30}$ automaton.
Further results with more layers of the $C$ diagonal SSM with nonlinear readout and $B\neq0$ are provided in Appendix~C. Introducing more layers did not improve the model's length generalization.
In contrast, \name achieves perfect length generalization on both tasks.

\subsection{Theoretical Characterization of Diagonal Selective SSMs}

Various recent works have derived different bounds on the computational capacity of diagonal selective SSMs~\cite{merrill_illusion_2024, sarrof2024expressivecapacitystatespace}.
However, an explanation for the behavior on commutative vs. non-commutative FSAs, as shown in Table~\ref{tab:C2xC30_vs_D30}, is missing.
We present an analysis of systems with diagonal transition matrices using a restrictive assumption. The assumption is that models implement a mapping consistent with the one described in Sec.~\ref{sec:background}, for which the $B$ matrix is irrelevant.
Single-layer diagonal selective SSMs that do not utilize the $B$ matrix are restricted to commutative automata:

\begin{proposition}
    Given a sequence of inputs $(u_1,...,u_T)$, let the transition matrices $(A(u_1),...,A(u_T))$ be simultaneously diagonalizable. Under the described mapping of FSA to a single-layer selective SSMs which sets $x_0=enc(q_{\text{init}})$, $b(u_t)=0$, and whose transition matrices are simultaneously diagonalizable, the selective SSM can only emulate commutative automata.
\end{proposition}

\begin{proof}
If the selective SSM is parametrized according to the mapping shown in Sec.~\ref{sec:background}, then it is equivalent to the following system:
\begin{equation}
\label{eq_slssm_no_input}
    x_{t+1} = A(u_t)x_t
\end{equation}
We assumed that the matrices $A(u_t)$ are simultaneously diagonalizable. 
This means that there exist a single invertible matrix $W\in\mathbb{C}^{n \times n}$ such that all transition matrices can be expressed as $A(u_t) = W\Lambda(u_t)W^{-1}$, with the diagonal matrix $\Lambda(u_t) \in\mathbb{C}^{n \times n}$. 
If we insert the above decomposition into the reduced system $x_{t+1} = A(u_t)x_t$, we obtain the form $x_{t+1} = W\Lambda(u_t)W^{-1}x_t$. 
The dynamics of this system are unchanged if we change the representation basis by multiplying the system from the left with $W^{-1}$.
By setting $\tilde{x}_t = W^{-1}x_t$, we see that the above system is equivalent to the diagonal system $\tilde{x}_{t+1} = \Lambda(u_t)\tilde{x}_t$. 
Therefore, if we interpret Eq.~\eqref{eq_slssm_no_input} as implementing the dynamics of an FSA, if this system admits an equivalent diagonal representation then the final automaton state is invariant to the order in which the inputs are presented.

\end{proof}

The mapping we impose is reminiscent of several other mappings from literature~\cite{merrill_illusion_2024,liu_transformers_2023}. 
In fact, if a model based on Eq.~\eqref{eq_slssm} implements a mapping different from the one we describe in Sec.~\ref{sec:background}, then it is not necessarily commutative. This can be seen by unrolling Eq.~\eqref{eq_lssm} for several time steps.
 \begin{align*}
 x_1 &= A_1x_0 + b(u_1) \\
 x_2 &= A_2A_1x_0 + A_2b(u_1) + b(u_2) \\
 x_3 &= A_3A_2A_1x_0 + A_3A_2b(u_1) + A_3b(u_2) + b(u_3)
 \end{align*}
 with the abbreviation $A(u_t) =: A_t$. 
Since the product of diagonal matrices commutes, it is exactly the terms containing $b(u_t)$ that break the commutativity.
However, while the model that utilizes the $B$ matrix learns to emulate the non-commutative automaton in our experiments, it only generalizes to a limited degree.

\section{Related Work}

\subsection{State-Space Models}

Early SSMs build on the HiPPO theory of optimal projections~\cite{gu_hippo_2020}.
The S4 model~\cite{gu_efficiently_2022} is an early example of an SSM used in a deep neural network, and it significantly advanced the state-of-the-art on a collection of long-range modeling tasks compared to the Transformer.
Diagonal SSMs emerged from a desire for more efficient parallelizable computation in the form of DSS~\cite{gupta_diagonal_2022} and S4D~\cite{gu_parameterization_2022}. 
S5 introduces effective simplified MIMO SSMs~\cite{smith_simplified_2023}.
The LRU~\cite{orvieto_resurrecting_2023} is a simplified and effective linear SSM utilizing complex-valued transition matrices.
The H3~\cite{fu_hungry_2023} presents advancements towards realistic language modeling using SSMs, but shows that such models perform best when interleaved with attention layers.
Mamba~\cite{gu_mamba_2023} is the first selective SSM to outperform the Transformer~\cite{vaswani_attention_2017} in a range of important NLP tasks including language modeling.
\cite{fan_advancing_2024} presents a block-diagonal selective SSM which achieves perfect length generalization on three out of the four regular language tasks from~\cite{deletang_neural_2023}.
Compared to previous work, we are the first to demonstrate that all finite-state automata from~\cite{deletang_neural_2023}, and others from~\cite{liu_transformers_2023}, can be emulated with single layer selective SSM utilizing a linear readout. We additionally provide experimental results with various single- and multi-layer complex-valued diagonal SSMs on FSA emulation.

\subsection{Formal Analysis of Sequence Models}

The ability of neural networks to model various formal models of computation is a long-standing area of research~\cite{siegelmann_computational_1995, minsky_computation_1967}.
One of the first models studied was the RNN, which can implement the dynamics of any FSA, with \cite{svete2023efficient} reviewing three different exact mappings of FSA to RNNs.
The presented mappings are due to~\cite{minsky_thesis_1954, dewdney_rnn_1977, indyk_optimal_1995}.
Recently, many such studies of the Transformer model have emerged. A survey of various bounds on the Transformer's expressiveness can be found in \cite{strobl_what_2024}. Particularly interesting is the study due to~\cite{merrill_parallelism_2023}, which conjectures that any model architecture as parallelizable as the Transformer will obey limitations similar to it. 
\cite{deletang_neural_2023} presents experimental study of the Transformer~\cite{vaswani_attention_2017}, RNN~\cite{elman_finding_1990}, LSTM~\cite{hochreiter_long_1997}, Stack-RNN \cite{mikolov_stack_2015}, Tape-RNN~\cite{suzgun_memory_2019} and other architectures in formal language transduction. 
We extend their analysis by considering SSM-based architectures.
\cite{merrill_illusion_2024} derives a bound on diagonal selective SSMs with logarithmic precision representation, placing them in the $TC^0$ circuit complexity class. This complexity class encompasses the presented $C_n \times C_m$ and $D_n$ groups. Their experimental results do not evaluate the length-generalization aspect of SSMs and selective SSMs.
\cite{sarrof2024expressivecapacitystatespace} show that a stack of complex diagonal SSM layers can emulate any automaton in $TC^0$. Their experimental results only evaluate Mamba and the Transformer, and the commutativeness of the automata is not a central aspect of their work.

\section{Conclusion}

In this work, motivated by the inability of a wide range of sequence models to emulate arbitrary automata, we have presented \name. 
It utilizes a dictionary of \emph{dense} transition matrices, combined at each time step using a softmax selection mechanism and operator normalization, and a readout which consists of layer normalization followed by a linear map.
\name is the first selective state-space model to achieve perfect length generalization on a diverse set of FSA emulation tasks using a single layer.

We then evaluated more efficient selective SSMs with diagonal complex valued transition matrices on a set of FSA emulation tasks. 
We observed that they exhibit significantly worse length generalization than their dense counterparts.
We probed deeper into this result by investigating their performance on two similar automata which differ in one crucial property: commutativity with respect to the inputs.
Our experimental analysis confirms that diagonal selective SSMs exhibit a significantly higher degree of length generalization on the commutative automaton compared to the non-commutative one.
We explain the results by drawing a connection between a general mapping of FSA dynamics onto selective SSM weights and linear system theory.
Assuming that the selective SSMs do not implement an unintuitive mapping of FSA dynamics, we observe that they indeed cannot model non-commutative automata.

We list some potential avenues for future work. Firstly, \name's softmax selection mechanism allows the use of temperature scaling and annealing strategies, which could lead to more interpretable and efficient models. Secondly, general mappings of $TC^0$ non-commutative automata to diagonal selective SSMs can be investigated further. Finally, the model could be evaluated on more natural data to reveal whether the increased formal expressiveness translates to other real-world applications.

\section*{Acknowledgments}
We would like to thank Nicolas Menet for providing inputs on the theoretical discussion of this work.

\bibliography{sequence-modelling}

\clearpage
\appendix
\renewcommand\thefigure{\thesection.\arabic{figure}}
\renewcommand\thetable{\thesection.\arabic{table}}   

\section{Task Description}
\label{apx:tasks}

\subsection{Parity}

\emph{Parity} is one of the simplest regular languages with $Q=\{\text{Even},\text{Odd}\}$ and $\Sigma=\{0,1\}$. Given a sequence of binary inputs, the corresponding output is \emph{Even} if the number of ones in the sequence is even, otherwise it is \emph{Odd}. The automaton starts in state $q_{\text{init}}=\text{Even}$.  

\subsection{Even Pairs}

The \emph{Even Pairs} task consists of determining whether the number of $01$ and $10$ substrings in a longer binary string is equal. For example, given a string 0101, it contains two 01 substrings and one 10 substring. As the number of these substrings is not equal, this string is not part of the Even Pairs regular language. On the other hand, 01000 contains one 01 substring and one 10 substring, meaning that this string is part of the language. This task effectively reduces to detecting whether the start and end of the string are equal symbols, which might explain the relatively high accuracy that certain models obtain in this task.

\subsection{Cycle}

Originally called \emph{Cycle Navigation}, this is an FSA consisting of five enumerated states $Q=\{1,2,3,4,5\}$ and the input alphabet $\Sigma=\{L,R,S\}$, standing for \emph{Left}, \emph{Right}, and \emph{Stay}. It is illustrated in Figure~\ref{fig:cycle}.

\begin{figure}[h!]
 \centering
 \includegraphics[width=0.4\linewidth]{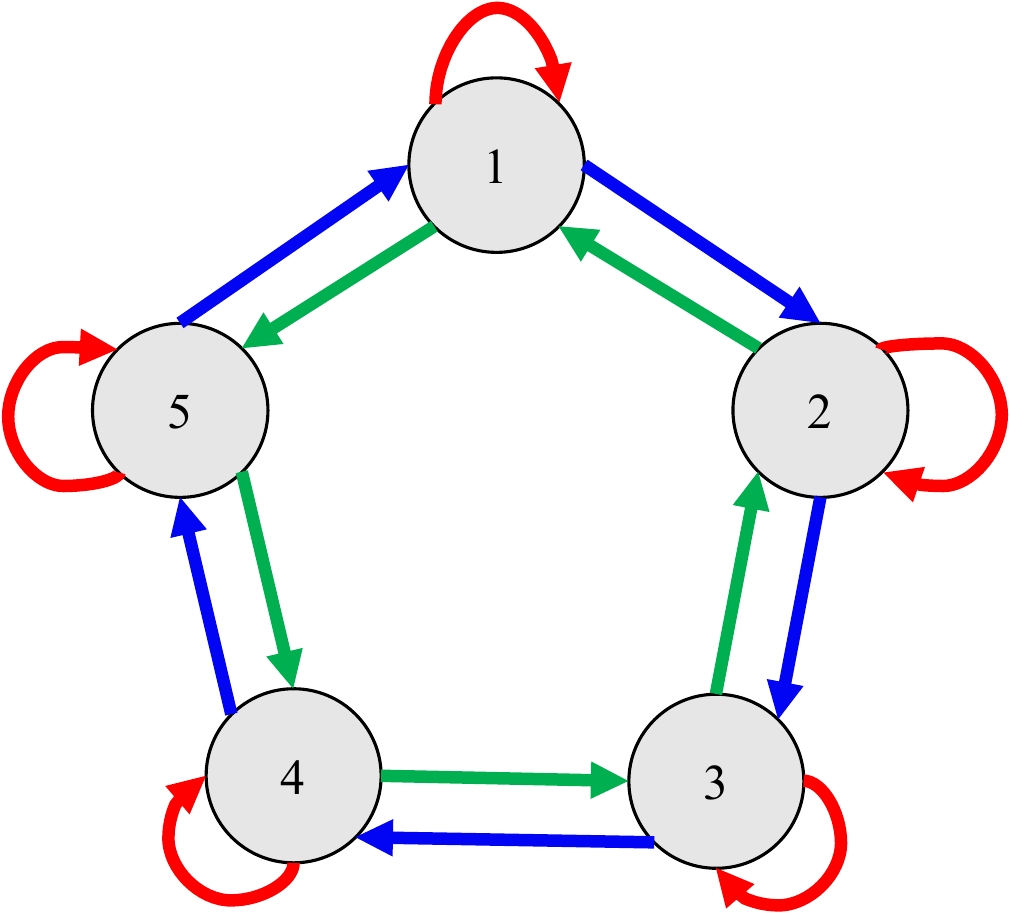}
 \caption{The \emph{Cycle} automaton starts in state $q_{\text{init}}=1$. The state graph is traversed using three actions, \emph{Left}, \emph{Right}, and \emph{Stay}. In the given sketch, these three actions respectively correspond to the green, blue, and red arrows.}
 \label{fig:cycle}
 \end{figure}

\subsection{Arithmetic}

The inputs of the \emph{Arithmetic} task are digits between 0 and $4$ interleaved with the operations $+, -$ and $*$. The task is to compute the result modulo 5. For example, 2 * 4 + 1 - 2 evaluates to 2. Obtaining the final state of the automaton thus reduces to evaluating an expression in modular arithmetic.

\subsection{Direct Product of Cyclic Groups $C_2$ and $C_n$}

Having reviewed the definitions of the FSAs from~\cite{deletang_neural_2023}, we now explain the set of FSAs derived from the structure of different algebraic groups. As already mentioned, these examples were generated using the code from~\cite{liu_transformers_2023}.

The direct product of the cyclic groups $C_2$ and $C_n$, denoted as $C_2 \times C_n$, is a solvable, commutative group~\cite{carter_visual_2009}. Recent results show that automata with solvable transformation groups can be emulated by diagonal selective SSMs~\cite{merrill_illusion_2024, sarrof2024expressivecapacitystatespace}. On the contrary, automata with nonsolvable transformation groups cannot be emulated by diagonal selective SSMs~\cite{merrill_illusion_2024}.

The Cayley diagram of the $C_2 \times C_5$ group is shown in Figure~\ref{fig:c25}.

\begin{figure}[h!]
 \centering
 \includegraphics[width=0.4\linewidth]{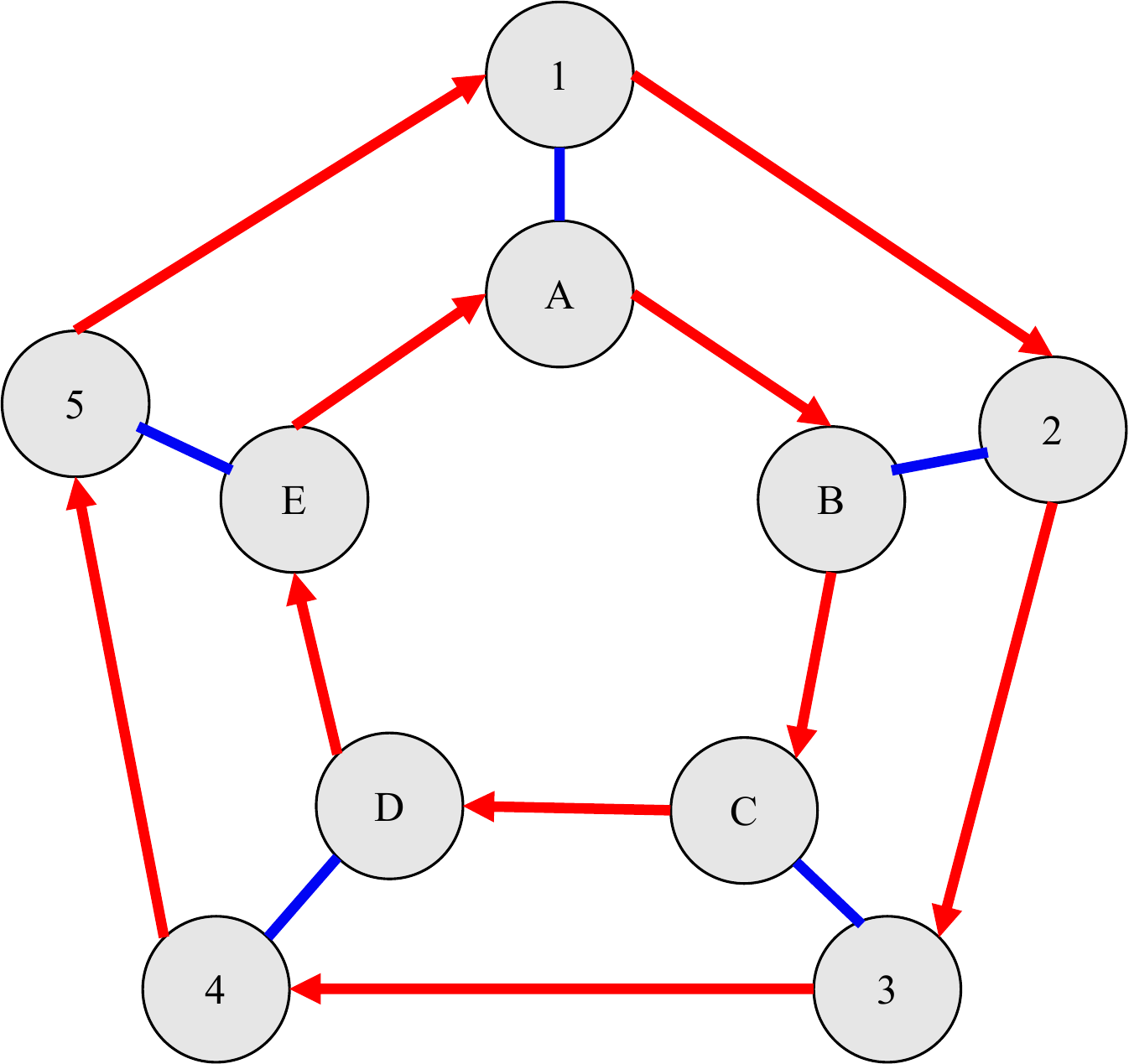}
 \caption{The $C_2 \times C_n$ group, visualized as a Cayley diagram~\cite{carter_visual_2009} with $n=5$. There are two actions, \emph{Move} on current cycle and \emph{toggle} between the two concentric cycles. Starting at any state, applying \emph{toggle} followed by \emph{move} results in the same final state as when applying \emph{move} followed by \emph{toggle}.}
 \label{fig:c25}
 \end{figure}

\subsection{Dihedral Group $D_n$}

The $D_n$ group consists of all symmetries of a regular $n$-gon~\cite{carter_visual_2009}. When visualized, its structure looks very similar to the $C_2 \times C_n$ group, with the exception that the two cycles now point in opposite directions.
This group is also solvable, but is crucially not commutative.

\begin{figure}[h!]
 \centering
 \includegraphics[width=0.4\linewidth]{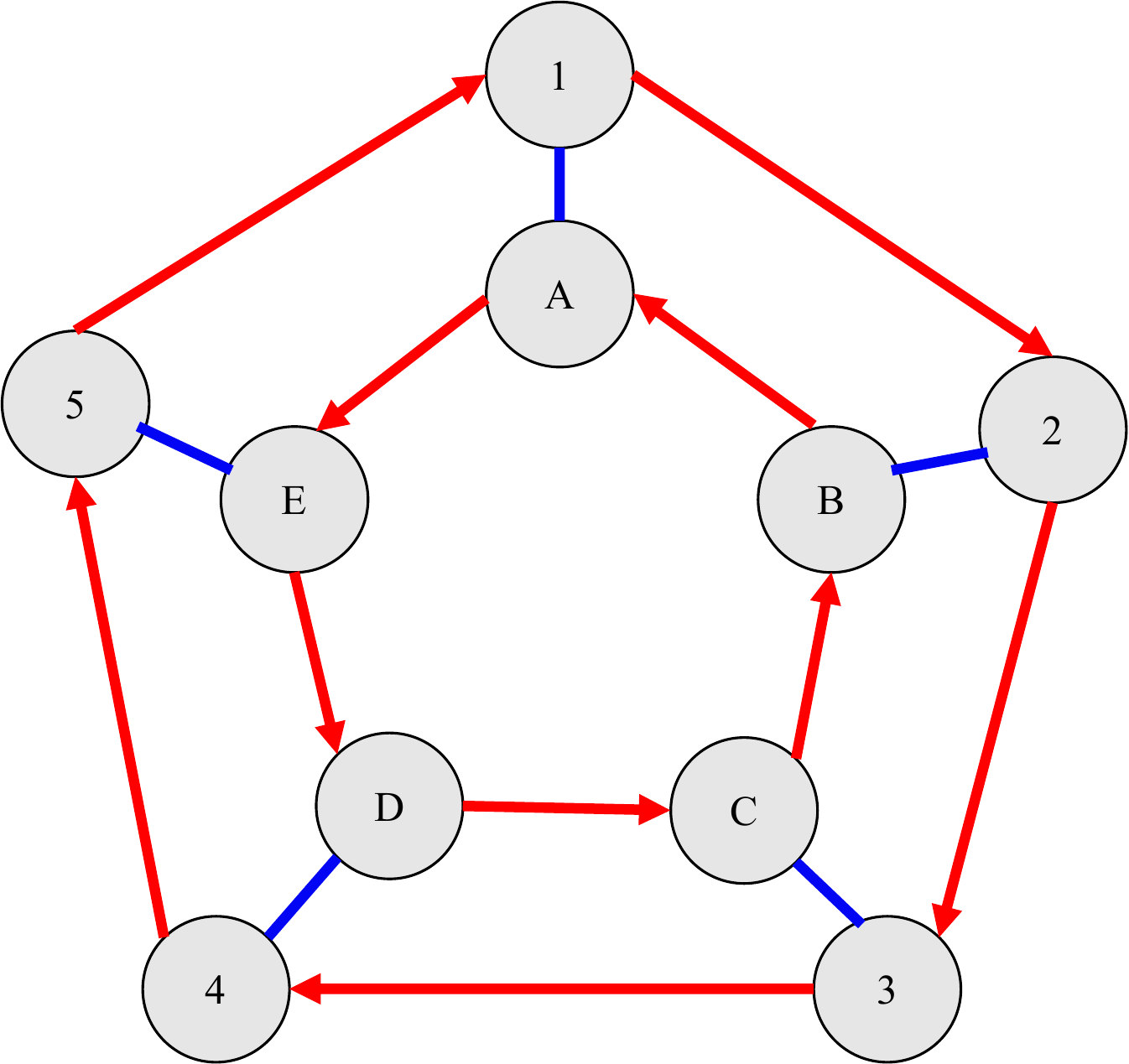}
 \caption{The $D_n$ automaton, visualized with $n=5$. There are two actions, \emph{Move} on current cycle and \emph{toggle} between the two concentric cycles. The FSA is not commutative. Starting at state 1, applying \emph{move} followed by \emph{toggle} results in state \emph{B}. However, starting from state 1, applying \emph{toggle} followed by \emph{move} results in state \emph{E}.}
 \label{fig:d5}
 \end{figure}

\subsection{Alternating Group $A_5$}

The $A_5$ group consists of all even permutations of a set of five elements, resulting in 60 states. Even permutations are those that can be obtained by applying an even number of transpositions, i.e., pairwise swaps of elements.
If we imagine five objects enumerated with numbers 0 to 4, the initial state can be represented by $(0,1,2,3,4)$. Starting from this state, every other state of the $A_5$ automaton can be reached by applying two actions: \emph{Swap} and \emph{Cycle}. \emph{Swap} swaps the positions of the first two pairs elements in the current state. For example, $(0,1,2,3,4)\xrightarrow[]{\text{swap}} (1,0,3,2,4)$. \emph{Cycle} cyclically shifts the elements to the right: $(0,1,2,3,4)\xrightarrow[]{\text{cycle}} (4,0,1,2,3)$. The permutations are enumerated according to the Python library sympy.combinatorics.permutations.
$A_5$ is the smallest nonsolvable group.
A visualization of the group can be found in~\cite{carter_visual_2009}.

\section{Experimental Setup}
\label{apx:experimental_setup}

Our experimental framework is based on the code by~\cite{deletang_neural_2023}. The code is modified to work with the PyTorch framework, and was extended by adding novel models and tasks. 
We used a modified version of the generative code from~\cite{liu_transformers_2023} to define the $C_2 \times C_n$, $D_n$, and $A_5$ FSAs.

\subsection{Experiment Hyperparameters}
Table~\ref{tab:tab_1_mamba_hyper} through Table~\ref{tab:tab_1_sdssm_hyper} show the hyperparameters required to reproduce our main experimental results. All models were trained using the Adam optimizer with $\beta_1=0.9$, $\beta_2=0.999$, batch size 128, and 1,000,000 training steps. For S4, S4D, Hyena, and Mamba results, we sweeped the learning rate in \{1e-5, 1e-4, 1e-3, 1e-2\}, weight decay in \{0, 1e-4, 1e-3, 1e-2\} and state size in \{128, 256\}.
\name, the $\mathbb{C}$ diagonal models, the RNN, and the LSTM were always trained without regularization.
Unless explicitly stated otherwise, the state size is 64.

\begin{table}[h!]
\centering
\begin{tabular}{|l|l|l|l|l|}
\hline
  & State & LR & WD \\ \hline
Parity Check & 128 & 1e-4 & 0.0\\ \hline
Even Pairs & 128 & 1e-2 & 0.0 \\ \hline
Cycle Navigation & 256 & 1e-4  & 0.0 \\ \hline
Modular Arithmetic & 256 & 1e-3 & 0.0 \\ \hline
\end{tabular}
\caption{Mamba hyperparameters for Table 1. The embedding size is 64.}
\label{tab:tab_1_mamba_hyper}
\end{table}

\begin{table}[h!]
\centering
\begin{tabular}{|l|l|l|l|}
\hline
 Task & LR & WD  \\ \hline
Parity Check  & 1e-4 & 0.0  \\ \hline
Even Pairs  & 2e-4 & 0.0 \\ \hline
Cycle Navigation   & 1e-4 & 0.0 \\ \hline
Modular Arithmetic   & 2e-4 & 1e-4  \\ \hline
$C_2 \times C_4$  & 5e-4 & 1e-4  \\ \hline
$D_4$ & 1e-4 & 1e-4   \\ \hline
$A_5$ & 1e-4 & 1e-4    \\ \hline
\end{tabular}
\caption{RegularLRNN~\cite{fan_advancing_2024} hyperparameters for Table 1. The embedding size and the state size are both 64. Block size is 8 in each case.}
\label{tab:tab_1_reglrnn_hyper}
\end{table}

\begin{table}[h!]
\centering
\begin{tabular}{|l|l|}
\hline
 & LR. \\ \hline
Parity Check & 5e-4 \\ \hline
Even Pairs & 1e-3 \\ \hline
Cycle Navigation & 5e-4 \\ \hline
Modular Arithmetic & 1e-4 \\ \hline
$C_2 \times C_4$ & 5e-3 \\ \hline
$D_4$ & 1e-4 \\ \hline
$A_5$ & 1e-3 \\ \hline
\end{tabular}
\caption{Learning rates for $\mathbb{C}$ Diagonal in Table 1. The state size is 64.}
\label{tab:tab_1_cdiag_hyper}
\end{table}

\begin{table}[h!]
\centering
\begin{tabular}{|l|l|l|l|l|}
\hline
 Task  & \emph{k} & LR  &  $l_p$ \\ \hline
Parity Check  & 8 & 1e-4  & 1.2 \\ \hline
Even Pairs   & 8 & 2e-5 & 1.4   \\ \hline
Cycle Navigation  & 8 & 2e-5 & 1.3   \\ \hline
Modular Arithmetic &  18  & 1e-4 & 1.2 \\ \hline
$C_2 \times C_4$ & 6 & 2e-5 & 1.3 \\ \hline
$D_4$  & 6 &  1e-4  & 1.2 \\ \hline
$A_5$  &  6  & 2e-5 & 1.3 \\ \hline
\end{tabular}
\caption{SD-SSM hyperparameters for Table 1. We denote the number of transition matrices $(A_1,...,A_k)$ with \emph{k}.}
\label{tab:tab_1_sdssm_hyper}
\end{table}

\begin{table}[h!]
\centering
\begin{tabular}{|l|l|l|l|}
\hline
Length & RNN & LSTM & \name   \\ \hline
 4 & 0.001 & 0.0025 & 0.0025  \\ \hline
 5 & 0.001  & 0.0025 & 0.01 \\ \hline
 6 & 0.001  & 0.0025 & 0.0075 \\ \hline
 7 & 0.0025  & 0.0025 & 0.001 \\ \hline
 8 & 0.0025  & 0.0025 & 0.0001 \\ \hline
\end{tabular}
\caption{Learning rates for RNN, LSTM and \name results in Table~\ref{tab:short-training-sequency-A5}. For \name, we used $k=6$ transition matrices, state size 64, and $l_{1.2}$ column norm.}
\label{tab:C230_D30_SDSSM}
\end{table}

\begin{table}[h!]
\centering
\begin{tabular}{|l|l|l|l|}
\hline
Task & $B=0$? & Readout & LR \\ \hline
\multirow{4}{*}{$C_2 \times C_{30}$} & \multirow{2}{*}{Yes} & Linear & 1e-3 \\ \cline{3-4} 
 &  & Nonlin. & 1e-3 \\ \cline{2-4} 
 & \multirow{2}{*}{No} & Linear & 1e-3 \\ \cline{3-4} 
 &  & Nonlin. & 1e-2 \\ \hline
\multirow{4}{*}{$D_{30}$} & \multirow{2}{*}{Yes} & Linear & 1e-4 \\ \cline{3-4} 
 &  & Nonlin. & 1e-4 \\ \cline{2-4} 
 & \multirow{2}{*}{No} & Linear & 5e-4 \\ \cline{3-4} 
 &  & Nonlin. & 5e-3 \\ \hline
\end{tabular}
\caption{Leraning rates for $\mathbb{C}$ diagonal results in Table~\ref{tab:C2xC30_vs_D30}. The state size is 64.}
\label{tab:C230_D30_complex}
\end{table}

\begin{table}[h!]
\centering
\begin{tabular}{|l|l|l|l|}
\hline
 Task & LR & $l_p$  \\ \hline
$C_2 \times C_{30}$ & 5e-5 & 1.1   \\ \hline
$D_{30}$ & 5e-5 & 1.15    \\ \hline
\end{tabular}
\caption{Hyperparameters for \name results in Table~\ref{tab:C2xC30_vs_D30}. We used $k=10$ transition matrices and state size 64.}
\label{tab:C230_D30_SDSSM}
\end{table}

\begin{table}[h!]
\centering
\begin{tabular}{|l|l|l|l|l|}
\hline
 Task                   & 1 Layer &2 Layers & 3 Layers & 4 Layers  \\ \hline
  $C_2 \times C_{30}$  & 1e-2 & 5e-4 & 5e-4 & 5e-4 \\ \hline
  $D_{30}$  & 5e-3 & 5e-4 & 1e-3 & 1e-4   \\ \hline

\end{tabular}
\caption{Learning rates for $\mathbb{C}$ diagonal results in Table~\ref{tab:c_multi}. The state size is 64. The model uses an MLP readout with intermediate size 128, and $B\neq 0$.}
\label{tab:C230_D30_SDSSM}
\end{table}

\subsection{Details on the Length Efficiency Analysis}

The experimental setup in the \emph{Length Efficiency Analysis} in Sec.~\ref{sec:SDSSM} differs from the experimental setup used in the rest of the experiments. As the input lengths are very short, in order for input-output examples to cover all states of the automaton, we chose the initial state uniformly at random. This stands in contrast with the other experiments, which always used a fixed initial state. 
The set of states $Q$ was encoded into a matrix $X \in \mathbb{R}^{|Q|\times d}$ with $d=512$ and each element of the matrix $X$ randomly generated i.i.d. as $X_{i,j} \sim \mathcal{N}(0,1/\sqrt d)$. This ensures that each row of $X$ is in expectation of unit $l_2$ norm and that the rows are highly likely to be orthogonal to each other.
The initial state is chosen uniformly at random, and is projected to a lower dimensionality (128 in our case) using a trainable linear projection. The state is provided as the initial state of the model, $x_0$, after which the model is emulated for a short number of steps by consuming a randomly generated input sequence. The final state of the model, $x_T$, is projected to 512 dimensions using a linear layer $M$, and the matrix-vector product $X(Mx_T)$ defines the output logits of the model.
The model is trained to minimize the cross-entropy loss between the output logits and the true final state of the automaton. 
As the training sequences are very short, we observed overfitting. For this reason, we validate the models on sequences up to length 40 and report the accuracy obtained on sequences up to length 500.
This stands in contrast with the experimental setup used for all other experiments, in which we evaluate the final model after a fixed number of training steps as in~\cite{deletang_neural_2023}.

\section{Further Results}
\label{apx:further_results}

\subsection{Modern Sequence Models Fail to Emulate FSA With Two Layers}

We report results with two layers of S4, S4D, Hyena, H3, and Mamba on the set of regular language tasks from~\cite{deletang_neural_2023} in Table~\ref{tab:2_layer_best}. While the results are often better with two layers, the models still do not exhibit significant length generalization on this set of tasks.

\begin{table*}[t!]
\centering
\begin{tabular}{ccccccccc}
\toprule
Layers & Task        & S4    & S4D   &   H3      & Hyena     &     Mamba & RegularLRNN & \name \\    \cmidrule(r){1-1} \cmidrule(r){2-2} \cmidrule(r){3-3} \cmidrule(r){4-4} \cmidrule(r){5-5} \cmidrule(r){6-6} \cmidrule(r){7-7} \cmidrule(r){8-8} \cmidrule(r){9-9}
  \multirow{4}{*}{1} &  Parity      & ---  & 50.1 / 50.0 & 50.0 / 50.0      &  50.1 / 50.0    &  50.3 / 50.1   &   100\space/\space 100      & 100\space/\space 100 \\
 &  Even Pairs                      & ---   & 50.4 / 50.3 & 51.0 / 50.5   &  99.9 / 79.3   &  100\space/\space 100    &   100\space/\space 100      & 100\space/\space 100  \\
  &  Cycle                          & ---   &  33.6 / 29.2  & 20.1  / 20.0 & 20.1 / 20.0   & 21.1 / 21.0   &    100\space/\space 100     & 100\space/\space 100  \\
 &  Arithmetic                      & --- & 20.1 / 20.0 & 20.1 / 20.1   & 20.1   / 20.1    & 20.1 / 20.1   &      33.3 / 30.2   & 99.9 / 98.5 \\ 
 \cmidrule(r){1-1} \cmidrule(r){2-2} \cmidrule(r){3-9}
   \multirow{4}{*}{2} &  Parity      & 50.0 / 50.0  & 50.1 / 49.9 & 50.1 / 50.0      &  50.1 / 50.0    &  55.6 / 53.8   &  ---       & --- \\
 &  Even Pairs                      & ---   &       67.2 / 60.1  & 63.8  /  59.8   &  76.2  / 67.6   &  100 / 100    &  ---       & ---  \\
  &  Cycle                           & ---   &       50.3 / 50.1   & 20.0  / 19.9 & 20.2   / 19.9   & 59.5 / 40.9   &   ---      & --- \\
 &  Arithmetic                      & 20.1 / 20.1 & 21.5 / 20.5 & 20.2  / 20.0    & 20.2   / 20.0    & 23.5 / 23.0   &  ---      & --- \\
 
 \bottomrule       
\end{tabular}
\caption{Maximum/Average length generalization accuracy (\%) on sequences up to length 500 over 3 random seeds. Various sequence models with 1 and 2 layers. The models were trained on sequences up to length 40.}
\label{tab:2_layer_best}
\end{table*}



\subsection{SD-SSM with Nonlinear Readout on Arithmetic}

The \name with a nonlinear readout replaces the linear map in Figure~\ref{fig:model_sketch} with a two-layer MLP with the ReLU activation function. The intermediate size of the MLP is equal to the state size (64). We ran an extensive hyperparameter search over different grids before concluding that the MLP readout negatively impacts single-layer \name's performance on Arithmetic. 

%
We initially started a hyperparameter search with 6 transition matrices, learning rate in \{1e-4, 5e-4\} and weight decay in \{0, 1e-4, 1e-3, 1e-2\}. The best result of 63.9\%  was obtained with lr=1e-4 and wd=1e-3. Notably, increasing the weight decay factor past 1e-3 resulted in a drastic performance drop to 20.6\%. Without weight decay, the best result we could obtain was 29.2\%.
In the second hyperparameter search, we ran experiments with the number of transition matrices in \{18, 36\}, learning rate in \{2e-5, 1e-4, 5e-4\} and weight decay in \{0, 1e-4, 1e-3\}. The best result (71.9\%) was obtained with $k=18$ transition matrices, lr = 1e-4 and wd = 1e-3. Increasing the weight decay further, to 1e-2, again resulted in a drastic performance drop to 21.5\%.
In the third search, we increased the learning rates and used $k=10$ transition matrices. Concretely, the learning rate was set to values in \{0.001, 0.0025, 0.005\} and weight decay was set to values in \{1e-5, 1e-4 and 1e-3\}. In this search, $p$ was in \{1.15, 1.25, 1.3\}. The best result of 67.3\% was achieved in the centre of the grid, with lr=0.0025 and wd=1e-4.
Finally, we experimented with dropout on the hidden layer of the MLP readout. The hyperparameter search was run with 18 transition matrices, learning rate in \{2e-5, 1e-4, 5e-4\} and dropout in \{0.1, 0.2, 0.5\}, and resulted in the best accuracy of 57.3\% with lr=5e-4 and dropout=0.5

\subsection{Average Accuracies in the Length Efficiency Analysis}

In the \emph{Length Efficiency Analysis} in Sec.~\ref{sec:SDSSM}, we reported the best accuracy achieved over three random seeds in Table~\ref{tab:short-training-sequency-A5}. Here, we additionally report the average accuracy over three seeds in Table~\ref{tab:short-training-sequences-avg}. While in Table~\ref{tab:short-training-sequency-A5} we could observe that \name exhibits better length generalization when the best seed is selected, here we can see that it does exhibit higher variability across seeds compared to the the other models.

\begin{table*}[]
\centering
\begin{tabular}{cccccc}
\toprule
       & \multicolumn{5}{c}{Training Length}    \\  \cmidrule(r){2-6}
Model  &    4       &    5      &   6        & 7          & 8             \\ \cmidrule(r){1-1} \cmidrule(r){2-6}
RNN    &   7.1 / 6.3   &  26.1 / 21.5   & 85.4 / 78.3  &   99.4 / 98.3   &   99.9 / 99.5  \\ \cmidrule(r){1-1}
LSTM   &    14.7 / 11.1   &  66.9 / 45.7   &  97.6 / 95.4   &  99.9 / 99.9  &  100\space/\space 99.9    \\ \cmidrule(r){1-1}
SD-SSM (ours) &  31.5 / 15.0  &   83.3 / 41.3  &  97.2 / 92.1  &  99.4 / 95.0  &  100\space/\space 100 \\ 
\bottomrule          
\end{tabular}
\caption{Maximum/Average length generalization accuracy (\%) on sequences up to length 500 over three random seeds. The models were trained to emulate the $A_5$ automaton with very short sequences (4 to 8) using a state size of 128. \name used $k=6$ transition matrices.}
\label{tab:short-training-sequences-avg}
\end{table*}

\subsection{Results with Multiple Layers of the $\mathbb{C}$ Diagonal Model}

Finally, we report the best and average length generalization accuracy of the $\mathbb{C}$ diagonal model with 2, 3, and 4 layers on the $C_2 \times C_{30}$ and $D_{30}$ automata in Table~\ref{tab:c_multi}. In both tasks, with every number of layers investigated, 100\% in-domain accuracy could be achieved. However, we can see that increasing the number of layers beyond one tends to have a detrimental effect on the length generalization of the models.

\begin{table*}[t]

\centering
\begin{tabular}{lccccc}
\toprule
 & \multicolumn{4}{c}{Layers}    \\  \cmidrule(r){2-5}
Task        & 1 & 2 &  3    & 4    \\ \cmidrule(r){1-1} \cmidrule(r){2-2} \cmidrule(r){3-3} \cmidrule(r){4-4} \cmidrule(r){5-5} 
$C_2 \times C_{30}$  &  81.7 / 35.4  & 24.6 / 23.9 & 29.4 / 27.7   &    30.2 / 28.8     \\ \cmidrule(r){1-1}
$D_{30}$  & 61.0 / 28.6 &  29.2 / 28.4 & 24.7 / 20.3  &  25.9 / 23.9        \\\bottomrule       
\end{tabular}

\caption{Maximum/Average accuracy (\%)  on sequences up to length 600 over three random seeds of the $\mathbb{C}$ diagonal model with $B\neq 0$ and the nonlinear readout, with different model depth. The models were trained on sequences up to length 90.}
\label{tab:c_multi}
\end{table*}

\end{document}